\documentclass[twocolumn,11pt]{article}
\usepackage{multicol}
\usepackage[english]{babel}
\usepackage{amsmath}
\usepackage{amsthm}
\usepackage{amsfonts}

\newtheorem{definition}{Definition}
\newtheorem{lem}{Lemma}
\newtheorem{thm}{Theorem}

\newcommand{\norm}[1]{\lVert #1 \rVert}
\newcommand{\N}{\mathbb{N}}
\newcommand{\R}{\mathbb{R}}
\newcommand{\D}{\mathcal{D}}
\newcommand{\E}{\mathbb{E}}

\begin{document}

\title{On some theoretical limitations of Generative Adversarial Networks}

\author{Oriol, Benoît \\
        \texttt{benoit.oriol@polytechnique.edu} \\
       \'Ecole Polytechnique\\
       Palaiseau, 91120, France \\
       \and
       Miot, Alexandre \\
       \texttt{alexandre.miot@sgcib.com}\\
       Soci\'et\'e G\'en\'erale\\
       La D\'efense, 92800, France}
\twocolumn[
  \begin{@twocolumnfalse}
\maketitle
\begin{abstract}
Generative Adversarial Networks have become a core technique in Machine Learning to generate unknown distributions from data samples. They have been used in a wide range of context without paying much attention to the possible theoretical limitations of those models. Indeed, because of the universal approximation properties of Neural Networks, it is a general assumption that GANs can generate any probability distribution. Recently, people began to question this assumption and this article is in line with this thinking. We provide a new result based on extreme value theory showing that GANs can't generate heavy tailed distributions. The full proof of this result is given.
\end{abstract}
\end{@twocolumnfalse}
]

\section{Introduction}
The universal approximation property of neural networks (see \cite{Hornik1989} and \cite{Cybenko1989}) might make us assume that GANs can simulate any distribution from a Gaussian prior. However, neural networks, as functions are by design almost everywhere differentiable functions with bounded derivatives to limit exploding gradients phenomenons (see \cite{Pascanu2012}). By Rademacher (see \cite{Heinonen2005} for a proof) and mean value theorems, this is nearly equivalent to say that neural networks functions are Lipschitz continuous.
This fact basically sets the limitations of GANs to express any probability distribution given a Gaussian prior. The are numerous definitions of the concept of "fat", "longed" or "heavy" tailed distributions. They are usually not equivalent but all convey a sense of having a larger probability of being "big" compared to a Gaussian or Exponential distribution. Here we focus on two possible ways to define the concept. One, similarly to \cite{Seddik2020}, is focussing on finite samples and relies on classical concentration inequalities. The other is asymptotic and uses Extreme Value Theory to prove a new theorem in the continuity of the theoretical work of Huster et al. in \cite{Huster2021} and the experimental approach of \cite{Feder2020}.\\

\underline{Notations}, in the following, we make use of the following notations:
\begin{itemize}
    \item[-] $f$ is a Lipschitz function $\R^n \mapsto \mathbb{R}$, $||f||_l = \sup_{x,y \in \R^n, x\neq y} \frac{||f(x) - f(y)||}{||x-y||}$ its semi-norm,
    \item[-] $\gamma_n$ the Gaussian measure on $\R^n$,
    \item[-] $d$ the Euclidean distance in $\R^n$ i.e. $d(x,y)= ||y-x||$,
    \item[-] for any set $S\subset\R^n$ the $\epsilon$ neighbourhood $S_{\epsilon}= \{x\in\R^n \text{ such that } d(x, S)<\epsilon\}$, where $\epsilon > 0$,
    \item[-] $\bar{A}$ the complement of a subset $A\subset\R^n$,
    \item[-] $M$ the median of a mapping $\R^n \mapsto \R$ for the $\gamma_n$ measure i.e. $\gamma_n(\{f\geq M\})\geq \frac{1}{2}$ and $\gamma_n(\{f\leq M\})\geq \frac{1}{2}$,
    \item[-] $X$ a standard Gaussian random variable in $\R^n$,
    \item[-] $\bar{\Psi} = \frac{1}{\sqrt{2\pi}}\int_{x}^{+\infty} e^{-u^2/2} du$ the Gaussian tail function.
\end{itemize}
\section{Limitations through sub-gaussianity}
In this section we prove that given a Lipschitz function and a Gaussian prior $X$, $f(X)$ is sub-gaussian: a GAN with a Gaussian prior can only generate sub-gaussian distributions.
\begin{definition}
    A real valued random variable $Y$ is said to be sub-gaussian if it satisfies one of the two following equivalent properties :
    $$\Updownarrow 
    \begin{array}{l}
        \exists K\in\mathbb{R},\forall t\geq 0, P(\left|Y\right|\geq t) \leq 2 e^{-\frac{t^2}{K^2}} \\
        \exists K'\in\mathbb{R},\forall p \geq 1, ||X||_{L^p} \leq K' \sqrt{p}\text{ .}
    \end{array}
    $$
\end{definition}
For a proof see \cite{Vershynin2018}.

If $G=f(X)$, where $f$ is Lipschitz, by Lipschitz continuity and $\E(X)=0$
$$
\forall x\geq 0\quad  \mathbb{P}(|G-f(0)|\geq x) \leq  \mathbb{P}\left( |X|\geq \frac{x}{||f||_l} \right)\text{ .}
$$
In particular if $X$ is one dimensional then using a standard upper bound of the gaussian tail function $G$ will be sub-gaussian as a sum of two independant sub-gaussian functions, considering $f(0)$ as a constant random variable. \\

If $X$ was $n$ dimensional then,
$$
\mathbb{P}\left( |X|\geq \frac{x}{||f||_l} \right)=\mathbb{P}\left( |X|^2\geq \frac{x^2}{||f||_l^2} \right)\text{ ,}
$$
and $|X|^2$ following a $\chi^2$ distribution with $n$ degrees of freedom the sub-gaussianity of $G$ would seem to be dependent on the dimension of $X$. Yet, this is not the case as stated by the following remarkable result:

\begin{thm}[Gaussian concentration theorem \cite{Tsirelson1976}, \cite{Sudakov1978} and \cite{Borell1975}]\label{thm:gaussianconcentration}
    Let $X$ be a standard gaussian random variable on $\mathbb{R}^n$ and $f$ a Lipschitz function then $f(X) - \E(f(X))$ is sub-gaussian. More precisely,
    \begin{equation*}
        \forall \epsilon>0, \mathbb{P}( | f(X) - \E(f(X)) | \geq \epsilon ) \leq 2 e^{- 2 \epsilon^2 / ||f||_l^2}\text{ .}
    \end{equation*}
\end{thm}

So in particular, a GAN with a Gaussian prior will not be able to generate any realistic samples even if trained on a "fat tailed" distribution. This is not the first time that concentration of measure gives some strong theoretical limits to machine learning methods see \cite{Fawzi2018} or \cite{Prescott2021} for a more recent paper. Limitations of GANs has also been explored from a different perspective in \cite{Wiese2019}.\\

This theorem is also true when we replace the mean $\E(f(X))$ by a median of $f(X)$. The original proof of the theorem can be found in \cite{Tsirelson1976}. The proof is quite technical, a more accessible one can be found in \cite{Bobkov1997} . To get a sense of the concentration of measure phenomenon we provide here a simple proof with the median.  \\

The gaussian concentration theorem is a consequence of:

\begin{thm}[Gaussian Isoperimetric Theorem \cite{Borell1975}] Let's $A$ be a Borel set in $\R^n$ and $H=\{x\in\R^n \text{  such that }  x_1 < a\}$ with $a\in\R$ such that $\gamma_n(A) = \gamma_n(H)$ then
                $$
                \forall \epsilon\geq 0 \; \gamma_n(A_{\epsilon}) \geq \gamma_n(H_{\epsilon})\text{ .}
                $$
\end{thm}

It is easily seen that $\gamma_n(H_{r})=\Psi(a+r)$ where $\Psi$ is the cumulative distributive function of the one dimensional standard Gaussian distribution.  \\
It is not obvious at first sight what is the link between this theorem and the Gaussian concentration theorem for Lipschitz functions. The link is made defining the following `isoperimetric function' for $a\in[0,1]$ and $\epsilon>0$

\begin{align*}
\eta_{a}(\epsilon) &= \sup_{A\text{ borel set }\subset\R^n}\{\gamma_n(\bar{A_\epsilon}) \;|\; \gamma_n(A)\geq a\}\\
&= 1 - \inf_{A\text{ borel set }\subset\R^n}\{\gamma_n(A_\epsilon) \;|\; \gamma_n(A)\geq a\}\text{ .}
\end{align*}

\begin{lem}
Let $f\,:\,\R^n \mapsto \R$ be a Lipschitz function and $M$ a median for the Gaussian measure, then
$$
\forall \epsilon>0 \quad \gamma_n(|f-M|>\epsilon) \leq \eta_{\frac{1}{2}} \left( \frac{\epsilon}{||f||_l} \right)\text{ .}
$$
\end{lem}

\begin{proof}
    Let $A=\{f\leq M\}$, $\epsilon >0$ and $x\in A_{\epsilon}$ then
    $$
    \exists y\in A \text{ such that } d(x,A)\leq d(x,y) < \epsilon
    $$
    so, $f$ being Lipschitz and $y\in A$
    $$
    |f(x) - f(y)|\leq ||f||_l \; \epsilon\text{ .}
    $$
    So, $f(x) \leq M + ||f||_l\;\epsilon$ i.e. $\{f\geq M + ||f||_l\;\epsilon\} \subset \bar{A_\epsilon}$.
    Changing $\epsilon \rightarrow \frac{\epsilon}{||f||_l}$ we have proved that for any Lipschitz function $f$ of median $M$
    $$
    \forall \epsilon >0 \quad \gamma_n(\{f> M + \epsilon\}) \leq \eta_{\frac{1}{2}}\left(\frac{\epsilon}{||f||_l}\right)\text{ .}
    $$
    Noticing that $f$ if Lipschitz iif $-f$ is, $||f||_l=||-f||_l$, if $M$ is a median of $f$ then $-M$ is a median of $-f$ and applying what we just proved to $-f$ the result follows.
\end{proof}
We can now prove the Gaussian concentration theorem.
\begin{proof}
Let $A$ be a Borel set such that $\gamma_n(A)\geq \frac{1}{2}$, there exists a half-space $H=\R^{n-1} \times ]-\infty, a[$ such that $\gamma_n(A)=\gamma_n(H)=\Psi(a)\geq\frac{1}{2}$. From the Isoperimetric Gaussian Theorem

\begin{align*}
\forall \epsilon\geq 0, \quad & \gamma_n(A_\epsilon)  \geq  \gamma_n(H_\epsilon) \\
& \gamma_n(A_\epsilon)  \geq \Psi(a+\epsilon)\geq \Psi(\epsilon)\text{ .}
\end{align*}

$\Psi$ being non decreasing and  $a\geq 0$ as $\Psi(a)\geq\frac{1}{2}$. Taking the infinum on the left side and noticing that the infinimum is reached for $H$
\begin{multline*}
\forall \epsilon>0\text{,}\\\inf_{A\text{ borel set }\subset\R^n}\{\gamma_n(A_\epsilon) \;|\; \gamma_n(A)\geq \frac{1}{2}\} = \Psi(\epsilon)\text{ .}
\end{multline*}

That is to say,
$$
\eta_{\frac{1}{2}} = 1-\Psi=\bar{\Psi}.
$$
\end{proof}
\section{Limitations through Extreme Value Theory}
In this section, we prove the main result of this paper: given a Lipschitz function and a Gaussian prior $X$, if $f(X)$ is in a domain of attraction of an extreme value distribution of parameter $\xi$ then $\xi\leq 0$. In particular, $f(X)$ can't be "heavy tailed". In fact, we prove the theorem for a wider range of distributions. \\
In the following, we use the notations of Extreme Value Theory that are introduced in \cite{Huster2021}.
\begin{thm}
\label{thm:xibound}
    Let $n \in \R^*$, $f: \R^d \rightarrow \R$ a $\mathcal{C}^1$ a.e Lipschitz function with semi norm $L=||f||_l$. Let $G_k^d$ be a real random variable of probability distribution function $g_{G_k^d}(x) \propto \norm{x}_2^k e^{-\frac{\norm{x}_2^2}{2}}$, where $k\in \N$. If $f(G_k^d)$ is in the domain of attraction of the extreme value distribution of parameter $\xi\in\R$, i.e $f(G_k^d) \in \D(\mathcal{H}_{\xi})$, then $\xi \leq 0$.
\end{thm}
\begin{proof}
\textbf{Case $d=1$.}
We prove by contradiction that $\xi \leq 0$. Supposing that $\xi > 0$, by theorem 8.a \cite{Balkema1974}, $\forall \,\gamma\in]0, \xi[$, $\E[f(G_k^1)^\gamma]$ is finite and 
\begin{align*}
c_\gamma &= \lim_{t \rightarrow \infty} \E \left[ \left( \frac{f(G_k^1)}{t} \right)^\gamma | f(G_k^1) > t\right] \\
&= \left(1 - \frac{\gamma}{\xi} \right)^{-1}\text{ .}
\end{align*}
Let $\gamma \in ]0,\xi[$. We are only interested in the behaviour of the previous integral when $t$ goes to $+\infty$ so we can suppose that $t > f(0) + 1$ and $t > \sqrt{|k-1| + \gamma}L + |f(0)|$. \\
$f^{-1}(]t,\infty[)$ is an open set of $\R$ by continuity.
Open intervals are a countable base of $\R$, so we can write $f^{-1}(]t,\infty[) = \bigcup_{i \in \mathcal{I}} ]a_i,b_i[$ where $\mathcal{I}$ is finite or countable. We can also suppose that any of those intervals are disjoints. So we can suppose that :
\begin{itemize}
\item[-] $f^{-1}(]t,\infty[) = \bigcup_{i \in \mathcal{I}} ]a_i,b_i[$ where $\mathcal{I}$ is finite or countable
\item[-] $0 \notin ]a_i, b_i[$, $a_i \neq 0$, $b_i \neq 0$ and $f$ is strictly positive on each interval
\item[-] $-\infty\leq a_0 < b_0 \leq a_1 < b_2 \leq \ldots < b_{m^*} \leq + \infty$, where $m^*$ is equal to  $m$ or $\infty$ according to $\mathcal{I}$ cardinality
\item[-] $\forall i \in \mathcal{I}\; f(a_i)=t$ if $a_i > -\infty$ by continuity of $f$
\item[-] $\forall i \in \mathcal{I}\; f(b_i)=t$ if $b_i < \infty$ by continuity of $f$
\item[-] $\forall x \in ]a_i, b_i[\; |x|> \frac{t-f(0)}{L}$ as $f$ is $L$-Lipschitz and $]a_i, b_i[ \subset f^{-1}(]t, +\infty[$. In particular, noting $t^*=\min(t, t-f(0))$, $|x|> \frac{t^*}{L}$
\end{itemize}
Also, we are only interested in $t\to +\infty$ so we can suppose $t^{*^2}>|k-1|L^2$.
If $\mathcal{I} = \emptyset$, the case is trivial: $\E \left[ \left( \frac{f(G_k^1)}{t} \right)^\gamma | f(G_k^1) > t\right]$ is not defined, which is a contradiction. \\
Otherwise, the conditional expectation is well defined and finite and we have:
\begin{multline} \label{expi}
    \E \left[ \left( \frac{f(G_k^1)}{t} \right)^\gamma | f(G_k^1) > t\right] = \\ \frac{\sum_{i \in \mathcal{I}} \int_{a_i}^{b_i} \left(\frac{f(x)}{t}\right)^\gamma |x|^n e^{-\frac{x^2}{2}}dx}{\sum_{i \in \mathcal{I}} \int_{a_i}^{b_i} |x|^n e^{-\frac{x^2}{2}}dx}\text{ .}
\end{multline}
Let $i \in \mathcal{I}$. For the numerator, integrating by part:
\begin{multline*} 
    \int_{a_i}^{b_i} \left(\frac{f(x)}{t}\right)^\gamma |x|^k e^{-\frac{x^2}{2}}dx = \\
    \left[-\frac{|x|^{k}}{x} \left(\frac{f(x)}{t}\right)^\gamma e^{-\frac{x^2}{2}} \right]_{a_i}^{b_i} + \\
    \int_{a_i}^{b_i} \underbrace{\left(\frac{k-1}{x^2} +  \gamma\frac{f'(x)}{xf(x)} \right)}_{M}\left(\frac{f(x)}{t}\right)^\gamma |x|^k e^{-\frac{x^2}{2}}dx
\end{multline*}
The first integrated part is equal to $ \left[-\frac{|x|^{k}}{x}  e^{-\frac{x^2}{2}} \right]_{a_i}^{b_i}$, as we have seen on the interval bounds either $f$ is equal to $t$ or $f=O_{\pm\infty}(x)$.
We can bound the $M$ term as $|x|\geq\frac{t^*}{L}$, $f(x)>t>0$ and $|f'|<L$,
$$
|M|\leq |k-1| \frac{L^2}{t^{*^2}}+\gamma \frac{L^2}{t^{*^2}}\text{ .}
$$
We deduce the following inequalities: \\
\begin{multline}  \label{numleq}
    \int_{a_i}^{b_i} \left(\frac{f(x)}{t}\right)^\gamma |x|^k e^{-\frac{x^2}{2}}dx \leq \\
    \left[-|x|^{k-1} e^{-\frac{x^2}{2}} \right]_{a_i}^{b_i} + \\
    \frac{(|k-1| + \gamma)L^2}{t^{*^2}} \int_{a_i}^{b_i} \left(\frac{f(x)}{t}\right)^\gamma |x|^k e^{-\frac{x^2}{2}}dx\text{ ,}
\end{multline}
\begin{multline}  \label{numgeq}
    \int_{a_i}^{b_i} \left(\frac{f(x)}{t}\right)^\gamma |x|^k e^{-\frac{x^2}{2}}dx  \geq \\
    \left[-|x|^{k-1} e^{-\frac{x^2}{2}} \right]_{a_i}^{b_i} - \\
    \frac{(|k-1| + \gamma)L^2}{t^{*^2}}  \int_{a_i}^{b_i} \left(\frac{f(x)}{t}\right)^\gamma |x|^k e^{-\frac{x^2}{2}}dx\text{ .}
\end{multline}
The denominator has still a dependance on $f$ through the domain of integration, so $|x|> \frac{t^*}{L}$ is still valid and similarly:
\begin{multline}  \label{denumleq}
    \int_{a_i}^{b_i} |x|^k e^{-\frac{x^2}{2}}dx  \leq \left[-|x|^{k-1} e^{-\frac{x^2}{2}} \right]_{a_i}^{b_i} +\\
    \frac{|k-1|L^2}{t^{*^2}} \int_{a_i}^{b_i} |x|^k e^{-\frac{x^2}{2}}dx\text{ ,}
\end{multline}
\begin{multline}  \label{denumgeq}
    \int_{a_i}^{b_i} |x|^k e^{-\frac{x^2}{2}}dx  \geq \left[-|x|^{k-1} e^{-\frac{x^2}{2}} \right]_{a_i}^{b_i} - \\
    \frac{|k-1|L^2}{t^{*^2}} \int_{a_i}^{b_i} |x|^k e^{-\frac{x^2}{2}}dx\text{ .}
\end{multline}
Combining equations (\ref{numgeq}) and (\ref{denumleq}) in (\ref{expi}), we have:
\begin{equation*}
    \E \left[ \left( \frac{f(G_k^1)}{t} \right)^\gamma | f(G_k^1) > t\right] \geq \frac{1 - \frac{|k-1|L^2}{t^{*^2}}}{1 + \frac{(|k-1| + \gamma)L^2}{t^{*^2}}}\text{ .}
\end{equation*}
And combining (\ref{numleq}) and (\ref{denumgeq}) in (\ref{expi}), as we chose $t > \sqrt{|k-1| + \gamma}L + |f(0)|$, we have:
\begin{equation*}
    \E \left[ \left( \frac{f(G_k^1)}{t} \right)^\gamma | f(G_k^1) > t\right] \leq \frac{1 + \frac{|k-1|L^2}{t^{*^2}}}{1 - \frac{(|k-1| + \gamma)L^2}{t^{*^2}}}
\end{equation*}
So $c_\gamma = \lim_{t \rightarrow \infty} \E \left[ \left( \frac{f(G_k^1)}{t} \right)^\gamma | f(G_k^1) > t\right]= 1$. \\
Assuming $f(G_k^1) \in \D(\mathcal{H}_\xi), \xi >0$, entails $c_\gamma = (1- \frac{\gamma}{\xi})^{-1}$ and $\gamma=0$.
We conclude that $\xi \leq 0$.

\textbf{Case $d\in \N^*$.}
We prove that $\xi\leq 0$ by contradiction. If $\xi > 0$ using theorem 8.a \cite{Balkema1974}, $\forall 0 < \gamma < \xi$, $\E[f(G_k^d)^\gamma]$ is finite and
\begin{align*}
c_\gamma &= \lim_{t \rightarrow \infty} \E \left[ \left( \frac{f(G_k^d)}{t} \right)^\gamma | f(G_k^d) > t\right] \\
&=\left(1 - \frac{\gamma}{\xi} \right)^{-1}\text{ .}
\end{align*}
Let $\gamma \in ]0,\xi[$ and $t \in \R_{+}^*$ such that $t > f(0_d)$ and $t > L\sqrt{|k+d-2| + \gamma}L + |f(0)|$. Using the hyperspherical coordinates, we introduce the operator $H: \mathcal{L}([0,\pi]^{d-2} \times [0,2\pi]) \mapsto \R$:
$$
    H(f) =  \alpha \int_0^{\pi} ... \int_0^{\pi} \int_0^{2\pi} \prod_{i=1}^{d-2} \sin(\theta_{i})^{d-i-1}  f(\theta) d\theta\text{,}
$$
with $\alpha$ the normalising term for the $G_k^d$ distribution and $d\theta=d\theta_1 ... d\theta_{d-1}$. \\
Then, we have:
\begin{multline*}
    \E \left[ \left(\frac{f(G_k^d)}{t}\right)^\gamma 1_{f(G_k^d)>t} \right] = \\ H\Bigl(\theta \mapsto  \underbrace{ \int_0^{+\infty} 1_{f_\theta(r)>t} \left(\frac{f_\theta(r)}{t}\right)^\gamma r^{k+d-1} e^{-\frac{r^2}{2}} dr}_{\E\left [\left(\frac{f_\theta(G_{k+d-1}^1)}{t}\right)^\gamma 1_{f_\theta(G_{k+d-1}^1)>t}\right]}\Bigr)
\end{multline*}
with $x = r (x_1,..., x_d)$ and
\begin{align*}
x_1 &= \sin(\theta_1) \\
x_2 &= \sin(\theta_1) \cos(\theta_2) \\
&\vdots\\
x_{d-1}&=\sin(\theta_1) \sin(\theta_2) \ldots \cos(\theta_{d-1}) \\ 
x_{d}&=\sin(\theta_1) \sin(\theta_2) \ldots \sin(\theta_{d-1})\text{ .}
\end{align*}
For $\theta=(\theta_1,...\theta_{d-1})$, $f_\theta: r \in \R_+ \mapsto f(rx_1,...,rx_d)$ is $L$-Lipschitz as $(x_1, \ldots, x_d)$ is on the unit sphere. Also, $f(0)= f_\theta(0)$. We can use the bounds from the 1-dimensional proof. We note:
$$
    M_+ = \frac{1 + \frac{|k+d-2|L^2}{t^{*^2}}}{1 - \frac{(|k+d-2| + \gamma)L^2}{t^{*^2}}}
$$
$$
    M_- = \frac{1 - \frac{|k+d-2|L^2}{t^{*^2}}}{1 + \frac{(|k+d-2| + \gamma)L^2}{t^{*^2}}}\text{ .}
$$
We obtain:
\begin{multline*}
    \E \left[ \left(\frac{f(G_k^d)}{t}\right)^\gamma 1_{f(G_k^d)>t} \right] \leq \\ M_+ H\left(\theta \mapsto \mathbb{P}(1_{f_\theta(G_{k+d-1}^1)>t}) \right)\text{ .}
\end{multline*}
That is to say,
\begin{multline*}
 \E \left[ \left(\frac{f(G_k^d)}{t}\right)^\gamma 1_{f(G_k^d)>t} \right] \leq 
  M_+ \mathbb{P}(f(G_k^d) > t)\text{ .}
\end{multline*}
\normalsize
Similarly, we have:
\begin{multline*}
    \E \left[ \left(\frac{f(G_k^d)}{t}\right)^\gamma 1_{f(G_k^d)>t} \right]  \geq M_- \mathbb{P}(f(G_k^d) > t)\text{ .}
\end{multline*}
Thus, we conclude similarly that $c_\gamma$ is well-defined, finite, and $c_\gamma = 1$ that is $\gamma = 0$, which is absurd as $\gamma>0$.
\end{proof}
\section{Conclusion and future work}
Because of the intrinsic Lipschitz characteristics of Neural Networks, GANs expressivity is limited. In particular, a Gaussian prior cannot be used to simulate heavy tailed distributions. In the EVT framework, the question of the existence of a tail index for the generated distribution, or the conditions for its existence, remains. A theoretical partial answer is given in \cite{Huster2021} for GANs with ReLU or Leaky-ReLU activation functions and a finite number of neurons. The general case is still an open question. Likewise, determining the thinnest tail prior being able to simulate samples exhibiting heavy tails is an important question needing further investigations.\\
Moreover, experimentally, the problem of training GANs with a heavy-tailed prior remains too. Indeed, with such priors GANs are hard to train and exhibit numerical instabilities.

\newpage
\bibliographystyle{plain}
\bibliography{ganlimit}

\begin{thebibliography}{10}

\bibitem{Balkema1974}
A.~A. Balkema and Laurens {de Haan}.
\newblock Residual life time at great age.
\newblock {\em The Annals of Probability}, 1974.

\bibitem{Bobkov1997}
Sergei Bobkov.
\newblock An isoperimetric inequality on the discrete cube, and an elementary
  proof of the isoperimetric inequality in gauss space.
\newblock {\em The Annals of Probability}, 1997.

\bibitem{Borell1975}
Christer Borell.
\newblock The brunn-minkowski inequality in gauss space.
\newblock {\em Inventiones mathematicae}, 1975.

\bibitem{Cybenko1989}
George~V. Cybenko.
\newblock Approximation by superpositions of a sigmoidal function.
\newblock {\em Mathematics of Control, Signals and Systems}, 1989.

\bibitem{Fawzi2018}
Alhussein Fawzi, Hamza Fawzi, and Omar Fawzi.
\newblock Adversarial vulnerability for any classifier.
\newblock {\em ArXiv}, 2018.

\bibitem{Feder2020}
Richard Feder, Philippe Berger, and George Stein.
\newblock Nonlinear 3d cosmic web simulation with heavy-tailed generative
  adversarial networks.
\newblock {\em Phys. Rev. D}, 2020.

\bibitem{Heinonen2005}
Juha Heinonen.
\newblock Lectures on lipschitz analysis.
\newblock {\em Rep. Dept. Math. Stat}, 2005.

\bibitem{Hornik1989}
Kurt Hornik.
\newblock Approximation capabilities of multilayer feedforward networks.
\newblock {\em Neural Networks}, 1989.

\bibitem{Huster2021}
Todd Huster, Jeremy Cohen, Zinan Lin, Kevin Chan, Charles Kamhoua, Nandi~O.
  Leslie, Cho-Yu~Jason Chiang, and Vyas Sekar.
\newblock Pareto gan: Extending the representational power of gans to
  heavy-tailed distributions.
\newblock In {\em Proceedings of the 38th International Conference on Machine
  Learning}, Proceedings of Machine Learning Research, 2021.

\bibitem{Pascanu2012}
Razvan Pascanu, Tom{\'{a}}s Mikolov, and Yoshua Bengio.
\newblock Understanding the exploding gradient problem.
\newblock {\em CoRR}, 2012.

\bibitem{Prescott2021}
Jack Prescott, Xiao Zhang, and David Evans.
\newblock Improved estimation of concentration under {$L_p$-}norm distance
  metrics using half spaces.
\newblock {\em International Conference on Learning Representations}, 2021.

\bibitem{Seddik2020}
Mohamed El~Amine Seddik, Cosme Louart, Mohamed Tamaazousti, and Romain
  Couillet.
\newblock Random matrix theory proves that deep learning representations of
  {GAN}-data behave as {G}aussian mixtures.
\newblock In {\em Proceedings of the 37th International Conference on Machine
  Learning}, Proceedings of Machine Learning Research, 2020.

\bibitem{Sudakov1978}
Vladimir Sudakov and Boris Tsirelson.
\newblock Extremal properties of half-spaces for spherically invariant
  measures.
\newblock {\em Journal of Soviet Mathematics}, 1978.

\bibitem{Tsirelson1976}
Boris Tsirelson, Ildar Ibragimov, and Vladimir Sudakov.
\newblock Norms of gaussian sample functions.
\newblock {\em Proceedings of the Third Japan — USSR Symposium on Probability
  Theory}, 1976.

\bibitem{Vershynin2018}
Roman Vershynin.
\newblock {\em High-Dimensional Probability: An Introduction with Applications
  in Data Science}.
\newblock Cambridge Series in Statistical and Probabilistic Mathematics.
  Cambridge University Press, 2018.

\bibitem{Wiese2019}
Magnus Wiese, Robert Knobloch, and Ralf Korn.
\newblock Copula {\& m}arginal flows: Disentangling the marginal from its
  joint.
\newblock {\em CoRR}, 2019.

\end{thebibliography}

\end{document}